\documentclass[english,a4paper,12pt]{article}

\usepackage{graphicx}
\usepackage{amsxtra, amsfonts, amssymb, amstext, amsmath, mathtools}
\usepackage{amsthm}
\usepackage{booktabs}
\usepackage{nicefrac}
\usepackage{xspace}
\usepackage{url}
\usepackage{graphics,color}
\usepackage[algo2e,ruled,vlined,linesnumbered]{algorithm2e}
\usepackage{wrapfig}

\newtheorem{theorem}{Theorem}
\newtheorem{lemma}[theorem]{Lemma}

\newcommand{\oea}{\mbox{${(1 + 1)}$~EA}\xspace}

\newcommand{\onemax}{\textsc{OneMax}\xspace}
\newcommand{\LO}{\textsc{Leading\-Ones}\xspace}
\newcommand{\leadingones}{\LO}

\newcommand{\jump}{\textsc{Jump}\xspace}
\newcommand{\oneminmax}{\textsc{OneMinMax}\xspace}
\newcommand{\cocz}{\textsc{COCZ}\xspace}
\newcommand{\lotz}{\textsc{LOTZ}\xspace}
\newcommand{\ojzj}{\textsc{OneJumpZeroJump}\xspace}
\newcommand{\onejumpzerojump}{\textsc{OneJumpZeroJump$_{n,k}$}\xspace}
\newcommand{\jumpnk}{\textsc{Jump$_{n,k}$}}\xspace

\DeclareMathOperator{\cDis}{cDis}
\DeclareMathOperator{\mutbeta}{MUT^{\beta}(\cdot)}
\DeclareMathOperator{\mutbetaarg}{MUT^{\beta}}

\let\originalleft\left
\let\originalright\right
\renewcommand{\left}{\mathopen{}\mathclose\bgroup\originalleft}
\renewcommand{\right}{\aftergroup\egroup\originalright}

\usepackage{hyperref}
\begin{document}

\title{A First Runtime Analysis of the NSGA-II on a Multimodal Problem}

\author{Benjamin Doerr
\and
Zhongdi Qu}
\date{}

\maketitle

\begin{abstract}
Very recently, the first mathematical runtime analyses of the multi-objective evolutionary optimizer NSGA-II have been conducted. We continue this line of research with a first runtime analysis of this algorithm on a benchmark problem consisting of two multimodal objectives. We prove that if the population size $N$ is at least four times the size of the Pareto front, then the NSGA-II with four different ways to select parents and bit-wise mutation optimizes the OneJumpZeroJump benchmark with jump size~$2 \le k \le n/4$ in time $O(N n^k)$. When using fast mutation, a recently proposed heavy-tailed mutation operator, this guarantee improves by a factor of $k^{\Omega(k)}$. Overall, this work shows that the NSGA-II copes with the local optima of the OneJumpZeroJump problem at least as well as the global SEMO algorithm.
\end{abstract}

{\sloppy 
\section{Introduction}

The mathematical runtime analysis of evolutionary algorithms (EAs) has contributed significantly to our understanding of these algorithms, given advice on how to set their parameters, and even proposed new algorithms~ \cite{AugerD11,DoerrN20,Jansen13,NeumannW10}. Most of the insights, however, have been obtained by regarding artificially simple algorithms such as the \oea, the fruit fly of EA research.

In contrast, the recent work~\cite{ZhengLD22} succeeded in analyzing the \emph{non-dominated sorting genetic algorithm II (NSGA\nobreakdash-II)}~\cite{DebPAM02}, the multi-objective EA (MOEA) most used in practice~\cite{ZhouQLZSZ11}. This line of research was immediately followed up in~\cite{BianQ22ppsn} and~\cite{ZhengD22gecco}. These three works, just like the vast majority of the theoretical works on other MOEAs, only regard multi-objective problems composed of unimodal objectives (see Section~\ref{sec:previous} for more details).

In this work, we continue the runtime analysis of the NSGA\nobreakdash-II with a first analysis on a problem composed of two multi-modal objectives, namely the \ojzj problem proposed in~\cite{DoerrZ21aaai}. This problem, defined on bit strings of length~$n$, is a natural multi-objective analogue of the single-objective \jump problem, which might be the multimodal problem most studied in single-objective runtime analysis. The \jump problem (and the two objectives of the \ojzj problem) comes with a difficulty parameter~$k \in [1..n] := \{1, \dots, n\}$, which is the width of the valley of low fitness around the global optimum. Consequently, typical hillclimbers at some point need to flip the right $k$ bits, which is difficult already for moderate sizes of~$k$. For the multi-objective \ojzj problem the situation is similar. Here the set of Pareto optima is not a connected set in the search space $\{0,1\}^n$, but there are solutions which can only be reached from other points on the Pareto front by flipping $k$ bits, which creates a challenge similar to the single-objective case. 

\emph{Our results:} We conduct a mathematical runtime analysis of the NSGA\nobreakdash-II algorithm on the \ojzj problem with jump sizes $k \in [2..\frac 14 n]$. We allow that $k$ is functionally dependent on~$n$ and let all asymptotic notation be with respect to $n$. We note that $k=1$ is not interesting since it gives the \oneminmax problem studied in~\cite{ZhengLD22}. Since the runtimes we observe, as in the single-objective case, are at least exponential in~$k$, the restriction $k \le \frac 14 n$, done mostly to avoid some not very interesting technicalities, does not exclude any particular relevant situation. 

As \emph{runtime}, we consider the number of fitness evaluations until the full Pareto front is covered by the parent population, that is, at least one individual for each Pareto-optimal objective value is contained in the parent population of the NSGA\nobreakdash-II. As in~\cite{ZhengLD22}, we assume that the population size $N$ of the NSGA\nobreakdash-II is sufficiently large, here at least four times the size of the Pareto front (since a population size equal to the Pareto front size does not suffice to find the Pareto front even of the simple \oneminmax problem~\cite{ZhengLD22}, this assumption appears justified). We regard the NSGA\nobreakdash-II with four different ways to select the parents (each individual once (``fair selection''), uniform, $N$ independent binary tournaments, and $N$ binary tournaments from two random permutations of the population (``two-permutation tournament scheme'')), with bit-wise mutation with mutation rate~$\frac 1n$, and, for the theoretical analyses, without crossover. For all these variants of the NSGA\nobreakdash-II, we prove an expected runtime of at most $(1+o(1)) K N n^k$ on the $\ojzj$ problem with jump size~$k$, where $K$ is a small constant depending on the selection method.
Hence for $N = \Theta(n)$, the NSGA\nobreakdash-II satisfies the same asymptotic runtime guarantee of $O(n^{k+1})$ as the (mostly relevant in theory) algorithm global SEMO (GSEMO), for which a runtime guarantee of $(1+o(1)) 1.5 e (n-2k+3) n^k$ was shown in~\cite{DoerrZ21aaai}.

Since it has been observed that a heavy-tailed mutation operator called \emph{fast mutation} can significantly speed up leaving local optima~\cite{DoerrLMN17}, we also regard the NSGA\nobreakdash-II with this mutation operator. Similar to previous works, we manage to show a runtime guarantee which is lower by a factor of $k^{\Omega(k)}$ (see Theorem~\ref{thm:heavy} for a precise statement of this result). This result suggests that the NSGA\nobreakdash-II, similar to many other algorithms, profits from fast mutation when local optima need to be left. 

This work is organized as follow. We give an overview of the state of the art in the following section. In Section~\ref{sec:prelim}, we describe the NSGA\nobreakdash-II algorithm and the \ojzj benchmark. The heart of this work is the mathematical runtime analysis in Section~\ref{sec:runtime}. A brief experimental analysis can be found in Section~\ref{sec:exp}. The last section contains a conclusion and some directions for future works.

\section{State of the Art}\label{sec:previous}
The mathematical runtime analysis of evolutionary algorithms is a small, but established research area. Its foundations were laid in the nineties of the last century in works like~\cite{Rudolph97,DrosteJW02} (noting that a decent number of sporadic results existed earlier). Runtime analyses for MOEAs follows soon thereafter~\cite{LaumannsTDZ02,Giel03}. Both in single-objective and in multi-objective evolutionary computation theory, the typical focus of the early works was on analyzing how very simple algorithms (such as randomized local search, the \oea, their multi-objective counterparts, the simple evolutionary multi-objective optimizer (SEMO), or the global SEMO) solve simple benchmark problems (such as \onemax and \leadingones or \oneminmax or \lotz). While the single-objective side soon moved on to more complex algorithms~\cite{JansenJW05,Witt06} and problems~\cite{ScharnowTW04,NeumannW07}, the progress on the multi-objective side was slower. In particular, apart from an early work on the hypervolume-based SIBEA~\cite{BrockhoffFN08} (with~\cite{NguyenSN15,DoerrGN16} as only follow-up works), it was only in~\cite{LiZZZ16,HuangZCH19,HuangZ20} that the decomposition-based MOEA/D was analyzed. 

The first mathematical runtime analysis of the NSGA\nobreakdash-II, the MOEA most used in practice, was conducted only very recently~\cite{ZhengLD22}. It showed that this algorithm can efficiently find the Pareto front of the \oneminmax and \lotz bi-objective problems when the population size $N$ is at least some constant factor larger than the size of the Pareto front (which is $n+1$ for these problems). In this case, once an objective value of the Pareto front is covered by the population, it remains so for the remaining run of the algorithm. 
This is different when the population size is only equal to the size of the Pareto front. Then such values can be lost, and this effect is strong enough so that for an exponential number of iterations a constant fraction of the Pareto front is not covered~\cite{ZhengLD22}. Nevertheless, also in this case the NSGA\nobreakdash-II computes good approximations of the Pareto front as the first experiments in~\cite{ZhengLD22} show.

This aspect is discussed in more detail in~\cite{ZhengD22gecco}. There it was observed that the main reason for the frequent loss of desirable solutions lies in the fact that the NSGA\nobreakdash-II does not update the crowding distance after each removal, but bases all removals in the selection of the next parent population on the same initial crowding distance values. When using an updated version of the crowding distance, also with smaller population sizes provably good approximations can be computed. 

The most recent work~\cite{BianQ22ppsn} extends~\cite{ZhengLD22} in several directions. (i)~For the NSGA\nobreakdash-II using crossover, runtime guarantees for the \oneminmax, \cocz, and \lotz problems are shown which agree with those in~\cite{ZhengLD22}. (ii)~By assuming that individuals with identical objective value appear in the same or inverse order in the sortings used to compute the crowding distance, the minimum required population size is lowered to $2(n+1)$. (iii)~A stochastic tournament selection is proposed that reduces the runtimes by a factor of $\Theta(n)$ on \lotz and $\Theta(\log n)$ on the other two benchmarks. 

The \oneminmax, \cocz, and \lotz benchmarks are all composed of two unimodal objectives, namely functions isomorphic to the benchmarks \onemax and \leadingones from single-objective EA theory. The theory of MOEA has largely focused on such benchmarks, a benchmark composed of multimodal objectives was only proposed and analyzed in~\cite{DoerrZ21aaai}. 

Besides the definition of the \ojzj problem, the main results in that work are that the SEMO algorithm cannot optimize this benchmark, that the GSEMO takes time $O((n-2k+3) n^k)$ (where the implicit constants can be chosen independent of $n$ and $k$), and that the GSEMO with fast mutation with power-law exponent $\beta>1$ succeeds in time $O((n-2k+3) k^{-k+\beta-0.5} n^k (\frac{n}{n-k})^{n-k})$ (where the implicit constant can be chosen depending on $\beta$ only). A slightly weaker bound (but still much better than that for classic GESMO) was shown for the GSEMO with the stagnation detection mechanism of Rajabi and Witt~\cite{RajabiW20}. Since the implementation of this stagnation detection mechanism in a MOEA is quite technical, we omit the details. 

\section{Preliminaries}\label{sec:prelim}
\subsection{The NSGA-II Algorithm}
We give an overview of the algorithm here and refer to \cite{DebPAM02} for a more detailed description of the general algorithm and to \cite{ZhengLD22} for more details on the particular version of the NSGA\nobreakdash-II we regard.

The algorithm starts with a random initialization of a parent population of size $N$. At each iteration, $N$ children are generated from the parent population via a mutation method, and $N$ individuals among the combined parent and children population survive to the next generation based on their ranks in a non-dominated sorting and, as tie-breaker, the crowding distance.

Ranks are determined recursively. All individuals that are not strictly dominated by any other individual have rank $1$. Given that individuals of rank $1,\dots,i$ are defined, individuals of rank $i+1$ are those only strictly dominated by individuals of rank $i$ or smaller. Clearly, individuals of lower ranks are preferred. 

The crowding distance, denoted by $\cDis(x)$ for an individual~$x$, is used to compare individuals of the same rank. To compute the crowding distances of individuals of rank $i$ with respect to a given objective function $f_j$, we first sort the individuals in ascending order according to their $f_j$ objective values. The first and last individuals in the sorted list have infinite crowding distance. For the other individuals, their crowding distance is the difference between the objective values of its left and right neighbors in the sorted list, normalized by the difference of the minimum and maximum values. The final crowding distance of an individual is the sum of its crowding distances with respect to each objective function.

At each iteration, the critical rank $i^*$ is the rank such that if we take all individuals of ranks smaller than $i^*$, the total number of individuals will be less than or equal to $N$, but if we also take all individuals of rank $i^*$, the total number of individuals will be over $N$. Thus, all individuals of rank smaller than $i^*$ survive to the next generation, and for individuals of rank $i^*$, we take the individuals with the highest crowding distance, breaking ties randomly, so that in total exactly $N$ individuals are kept.

\subsection{The \ojzj Benchmark}

The \ojzj benchmark was proposed in~\cite{DoerrZ21aaai} as a bi-objective analogue of the single-objective \jump benchmark~\cite{DrosteJW02}. This single-objective benchmark has greatly supported our understanding how evolutionary algorithms cope with local optima, see, e.g.,~\cite{JansenW02,DoerrLMN17,DangFKKLOSS18,RajabiW20,Doerr21cgajump,Doerr22}. 

Let $n\in\mathbb{N}$ and $k=[2..n/4]$. The bi-objective function $\onejumpzerojump = (f_1, f_2):\{0,1\}^n\rightarrow\mathbb{R}^2$ is defined by
\[f_1(x) = \begin{cases}
    k+|x|_1, & \text{if }|x|_1 \leq n-k\text{ or } x=1^n,\\
    n-|x|_1, & \text{else};
    \end{cases}\]
\[f_2(x) = \begin{cases}
    k+|x|_0, & \text{if }|x|_0 \leq n-k\text{ or } x=0^n,\\
    n-|x|_0, & \text{else}.
    \end{cases}\]
The aim is to maximize both $f_1$ and $f_2$. The first objective is the classical $\jumpnk$ function. It has a valley of low fitness around its optimum, which can be crossed only by flipping the $k$ correct bits, if no solutions of lower fitness are accepted. The second objective is isomorphic to the first, with the roles of zeroes and ones exchanged. 

According to Theorem~$2$ of \cite{DoerrZ21aaai}, the Pareto set of the $\onejumpzerojump$ function is 
\[S^*=\{x \in \{0,1\}^n \mid|x|_1 = [k..n-k]\cup\{0, n\}\},\] 
and the Pareto front $F^*$ is 
\[F^*= \{(a, 2k+n-a)\mid a\in[2k..n]\cup\{k, n+k\}\},\] 
making the size of the front $|F^*| = n-2k+3$.

We define the inner part of the Pareto set by $S_{I}^*=\{x\mid|x|_{1}\in [k..n-k]\}$, the outer part by $S_{O}^*=\{x\mid|x|_{1}\in\{0, n\}\}$, the inner part of the Pareto front by $F_{I}^*=f(S_{I}^*)=\{(a, 2k+n-a)\mid a\in[2k..n]\}$, and the outer part by $F_{O}^*=f(S_{O}^*)=\{(a, 2k+n-a)\mid a\in\{k, n+k\}\}$.
\section{Runtime Analysis for the NSGA-II}\label{sec:runtime}

In this section, we prove our runtime guarantees for the NSGA\nobreakdash-II, first with bit-wise mutation with mutation rate~$\frac 1n$ (Subsection~\ref{ssec:bitwise}), then with fast mutation (Subsection~\ref{ssec:fast}).

The obvious difference to the analysis for \oneminmax in~\cite{ZhengLD22} is that with  \ojzj, individuals with between 1 and $k-1$ zeroes or ones are not optimal. Moreover, all these individuals have a very low fitness in both objectives. Consequently, such individuals usually will not survive into the next generation, which means that the NSGA\nobreakdash-II at some point will have to generate the all-ones string from a solution with at least $k$ zeroes (unless we are extremely lucky in the initialization of the population). This difference is the reason for the larger runtimes and the advantage of the fast mutation operator. 

A second, smaller difference which however cannot be ignored in the mathematical proofs is that the very early populations of a run of the algorithm may contain zero individual on the Pareto front. This problem had to be solved also in the analysis of \lotz in~\cite{ZhengLD22}, but the solution developed there required that the population size is at least $5$ times the size of the Pareto front (when tournament selection was used). For \ojzj, we found a different argument to cope with this situation that only requires $N\ge 9$.

We start with a few general observations that apply to both mutation methods. A crucial observation, analogous to a similar statement in~\cite{ZhengLD22}, is that with sufficient population size, objective values of rank-$1$ individuals always survive to the next generation.

\begin{lemma}\label{lem:keeprank1} Consider one iteration of the NSGA\nobreakdash-II algorithm optimizing the \onejumpzerojump benchmark, with population size $N\geq 4(n-2k+3)$. If in some iteration $t$ the combined parent and offspring population $R_t$ contains an individual $x$ of rank~$1$, then the next parent population $P_{t+1}$ contains an individual $y$ such that $f(y) = f(x)$. Moreover, if an objective value on the Pareto front appears in $R_t$, it will be kept in all future iterations.
\end{lemma}

\begin{proof} Let $F_1$ be the set of rank-$1$ individuals in $R_t$. To prove the first claim, we need to show that for each $x\in F_1$, there is  a $y\in P_{t+1}$ such that $f(x)=f(y)$. Let $S_{1.1},\dots,S_{1.|F_1|}$ be the list of individuals in $F_1$ sorted by ascending $f_1$ values and $S_{2.1}, \dots ,S_{2.|F_1|}$ be the list of individuals sorted by ascending $f_2$ values, which were used to compute the crowding distances. Then there exist $a\leq b$  and $a'\leq b'$ such that $[a..b] = \{i \mid f_1(S_{1.i})=f_1(x)\}$ and $[a'..b'] = \{i \mid f_2(S_{2.i})=f_2(x)\}$. If any one of $a=1$, $a'=1$, $b=|F_1|$, or $b'=|F_1|$ is true, then there is an individual $y\in F_1$ satisfying $f(y)=f(x)$ of infinite crowding distance. Since there are at most $4 < N$ individuals of infinite crowding distance, $y$ is kept in $P_{t+1}$. So consider the case that $a, a' > 1$ and $b, b' < |F_1|$. By the definition of the crowding distance, we have that $\cDis(S_{1.a})\geq \frac{f_1(S_{1.a+1}) -  f_1(S_{1.a-1})}{f_1(S_{1.|F_1|})-f_1(S_{1.1})}\geq\frac{f_1(S_{1.a}) -  f_1(S_{1.a-1})}{f_1(S_{1.|F_1|})-f_1(S_{1.1})}$. Since $f_1(S_{1.a}) -  f_1(S_{1.a-1}) > 0$ by the definition of $a$, we have $\cDis(S_{1.a}) > 0$. Similarly, we have $\cDis(S_{1.a'}), \cDis(S_{1.b}), \cDis(S_{1.b'}) > 0$. For $i \in [a+1..b-1]$ and $S_{1.i}=S_{2.j}$ for some $j \in [a'+1..b'-1]$, we have that $f_1(S_{1.i-1})=f_1(x)=f_1(S_{1.i+1})$ and $f_2(S_{2.j-1})=f_2(x)=f_2(S_{2.j+1})$. So $\cDis(S_{1.i}) = 0$. Therefore, for each $f(x)$ value, there are at most $4$ individuals with the same objective value and positive crowding distances. By Corollary~$6$ in \cite{DoerrZ21aaai}, $|F_1|\leq n-2k+3$. So the number of rank-$1$ individuals with positive crowding distances is at most $4(n-2k+3)\leq N$ and therefore they will all be kept in $P_{t+1}$. 

The second claim then follows since if $x\in R_t$ and $f(x)$ is on the Pareto front, we have $x\in F_1$. By the first claim, $x\in P_{t+1}$ and therefore $x\in R_{t+1}$. The same reasoning applies for all future iterations.
\end{proof}

For our analysis, we divide a run of the NSGA\nobreakdash-II algorithm optimizing the \onejumpzerojump benchmark into the following stages.
\begin{itemize}
    \item \emph{Stage 1}: $P_t \cap S_{I}^* = \emptyset$. In this stage, the algorithm tries to find the first individual with objective value in $F_{I}^*$.
    \item \emph{Stage 2}: There exists a $v \in F_{I}^*$ such that $v \notin f(P_t)$. In this stage, the algorithm tries to cover the entire set $F_{I}^*$.
    \item \emph{Stage 3}: $F_{I}^* \subseteq f(P_t)$, but $F_{O}^* \nsubseteq f(P_t)$. In this stage, the algorithm tries to find the extremal values of the Pareto front.
\end{itemize}
By Lemma~\ref{lem:keeprank1}, once the algorithm has entered a later stage, it will not go back to an earlier stage. Thus, we can estimate the expected number of iterations needed by the NSGA\nobreakdash-II algorithm by separately analyzing each stage.

A mutation method studied in \cite{ZhengLD22} is to flip one bit selected uniformly at random. For reasons of completeness, we prove in the following lemma the natural result that the NSGA\nobreakdash-II with this mutation operator with high probability is not able to cover the full Pareto front of the \onejumpzerojump benchmark. 

\begin{lemma}\label{lem:1bit}
With probability $1 - N \exp(-\Omega(n))$, the NSGA\nobreakdash-II algorithm using one-bit flips as mutation operator does not find the full Pareto front of the \onejumpzerojump benchmark, regardless of the runtime.
\end{lemma}

\begin{proof}
Since $k \le n/4$, a simple Chernoff bound argument shows that a random initial individual is in $S_I^*$ with probability $1 - \exp(-\Omega(n))$. By a union bound, we have $P_0 \subseteq S_{I}^*$ with probability $1 - N \exp(-\Omega(n))$. We argue that in this case, the algorithm can never find an individual in $S_O^*$. 

We observe that any individual in $S^*_I$ strictly dominates any individual in the gap regions of the two objectives, that is, with between $1$ and $k-1$ zeroes or ones. Consequently, in any population containing at least one individual from $S^*_I$, such a gap individual can never have rank~$1$, and the only rank~$1$ individuals are those on the Pareto front. Hence if $P_t$ for some iteration $t$ contains only individuals on the Pareto front, $P_{t+1}$ will do so as well. 

By induction and our assumption $P_0 \subseteq S^*_I$, we see that the parent population will never contain an individual with exactly one one-bit. Since only from such a parent the all-zeroes string can be generated (via one-bit mutation), we will never have the all-zeroes string in the population. 
\end{proof}

In the light of Lemma~\ref{lem:1bit}, the one-bit flip mutation operator is not suitable for the optimization of \ojzj. We therefore do not consider this operator in the following runtime analyses.

\subsection{Runtime Analysis for the NSGA-II Using Bit-Wise Mutation}\label{ssec:bitwise}

In this section, we analyze the complexity of the NSGA\nobreakdash-II algorithm when mutating each bit of each selected parent with probability $\frac{1}{n}$. We consider four different ways of selecting the parents for mutation: (i)~fair selection (selecting each parent once), (ii)~uniform selection (selecting one parent uniformly at random for $N$ times), (iii)~via $N$ independent tournaments (for $N$ times, uniformly at random sample $2$ different parents and conduct a binary tournament between the two, i.e., select the one with the lower rank and, in case of tie, select the one with the larger crowding distance, and, in case of tie, select one randomly), and (iv)~via a two-permutation tournament scheme (generate two random permutations $\pi_{1}$ and $\pi_{2}$ of $P_t$ and conduct a binary tournament between $\pi_{j}(2i-1)$ and $\pi_{j}(2i)$ for all $i\in[1..N/2]$ and $j\in\{1,2\}$; this is the selection method used in Deb's implementation of the NSGA\nobreakdash-II when ignoring crossover \cite{DebPAM02}).

\begin{lemma}\label{lem:stage1} Using $N\ge 9$, bit-wise mutation for variation, and any parent selection method, stage $1$ needs in expectation at most $e(\frac{4k}{3})^k$ iterations.
\end{lemma}

\begin{proof}
Suppose $x$ is selected for mutation during one iteration of stage $1$ and $|x|_1 = i$. Then $i < k$ or $i > n-k$. If $i < k$, then the probability of obtaining an individual with $k$ $1$-bits is at least $\binom{n-i}{k-i}(\frac{1}{n})^{k-i}(1-\frac{1}{n})^{n-(k-i)} \geq (\frac{n-i}{n(k-i)})^{k-i}(1-\frac{1}{n})^{n-1} > \frac{1}{e}(\frac{3}{4(k-i)})^{k-i} \geq \frac{1}{e}(\frac{3}{4k})^k$ (where the second to last inequality uses the assumption that $i < k \leq \frac{n}{4}$). If $i > n-k$, then the probability of obtaining an individual with $n-k$ $1$-bits is at least $\binom{i}{i-(n-k)} (\frac{1}{n})^{i-(n-k)}(1-\frac{1}{n})^{2n-i-k} \geq (\frac{i}{n(i-n+k)})^{i-n+k}(1-\frac{1}{n})^{n-1} > \frac{1}{e}(\frac{3}{4(i-n+k)})^{i-n+k} \geq \frac{1}{e}(\frac{3}{4k})^k$ (where the second to last inequality uses the assumption that $i > n-k \geq \frac{3}{4n}$). Suppose $R_t$ is a combined parent and offspring population where there is at least one individual in $S^*_I$. Then the rank-$1$ individuals of $R_t$ are all in $S^*$. By the proof of Lemma~\ref{lem:keeprank1}, there are at most $8$ rank-$1$ individuals in $R_t$ that are not in $S^*_I$ and have positive crowding distance (4 copies of the all-zeroes string and 4 of the all-ones string). Since $N\ge 9$, at least one individual in $S^*_I$ survives to the next generation. Hence each iteration with probability at least $\frac{1}{e}(\frac{3}{4k})^k$ marks the end of stage $1$. Consequently, stage $1$ ends after in expectation at most $(\frac{1}{e}(\frac{3}{4k})^k)^{-1}=e(\frac{4k}{3})^k$ iterations.
\end{proof}

For the remaining two stages, we first regard the technically easier fair and uniform selection methods. 

\begin{lemma}\label{lem:stage2} Using population size $N\geq 4(n-2k+3)$, selecting parents using fair or uniform selection, and using bit-wise mutation for variation, stage $2$ needs in expectation $O(n\log n)$ iterations.
\end{lemma}

The proof, naturally, is very similar to the corresponding part of the analysis on \oneminmax~\cite{ZhengLD22}. 

\begin{proof}[Proof of Lemma~\ref{lem:stage2}] Let $x\in P_t$ be an individual such that $|x|_1=v\in [k..n-k]$. Denote the probability that $x$ is selected as the parent at least once by $p$, the probability that the result of mutating $x$ gives us a $y$ such that $|y|_1 = v+1$ by $p_{v}^{+}$ and the probability that the result of mutating $x$ gives us a $y$ such that $|y|_1 = v-1$ by $p_{v}^{-}$. Then, since by Lemma~\ref{lem:keeprank1}, $f(x)$ is kept in $f(P_t)$ for all iterations after it first appears, the expected number of iterations needed to obtain a $y$ such that $|y|_1 = v+1$ is at most $\frac{1}{pp_{v}^{+}}$. Similarly, the expected number of iterations needed to obtain a $y$ such that $|y|_1 = v-1$ is at most $\frac{1}{pp_{v}^{-}}$.

Consider one iteration $t$ of stage $2$. We know that there is $x\in P_t$ such that $|x|_1=v\in [k..n-k]$. Then the expected number of iterations needed to obtain objective values $(k+v+1, n+k-v-1), (k+v+2, n+k-v-2),\dots,(n, 2k)$ is at most $\sum_{i=v}^{n-k} \frac{1}{pp_{i}^{+}}$. Similarly, the expected number of iterations needed to obtain objective values $(k+v-1, n+k-v+1), (k+v-2, n+k-v+2),\dots,(2k, n)$ is at most $\sum_{i=k}^{v} \frac{1}{pp_{i}^{-}}$. As a result, the number of iterations needed to cover $F_{I}^*$ is at most $\sum_{i=v}^{n-k} \frac{1}{pp_{i}^{+}} + \sum_{i=k}^{v} \frac{1}{pp_{i}^{-}}$. With fair selection, we have $p=1$, and with uniform selection, $p = 1 - (1-\frac{1}{N})^N \geq 1-\frac{1}{e}$. Flipping each bit with probability~$\frac{1}{n}$, we have $p_{v}^{+} = \frac{n-v}{n}(1-\frac{1}{n})^{n-1} \geq \frac{n-v}{en}$ and $p_{v}^{-} = \frac{v}{n}(1-\frac{1}{n})^{n-1} \geq \frac{v}{en}$. So the expected number of iterations is at most $\frac{ne^2}{e-1}(\sum_{i=k}^{n-v} \frac{1}{i} + \sum_{i=k}^{v}\frac{1}{i})= O(n\log n)$.
\end{proof}

Different arguments, naturally, are needed in the following analysis of stage~3. 

\begin{lemma}\label{lem:stage3} Using population size $N\geq 4(n-2k+3)$ and bit-wise mutation for variation, stage $3$ needs in expectation at most $2en^k$ iterations if selecting parents using fair selection, and $2\frac{e^2}{e-1}n^k$ iterations if using uniform selection.
\end{lemma}

\begin{proof}
Consider one iteration $t$ of stage $3$. We know that there is an $x\in P_t$ such that $|x|_1 = k$. Denote the probability that $x$ is selected at least once to be mutated in this iteration by $p_1$. Conditioning on $x$ being selected, denote the probability that all $k$ $1$-bits of $x$ are flipped in this iteration by $p_2$. Then the probability of generating $0^n$ in this iteration is at least $p_{1}p_{2}$. Since by Lemma~\ref{lem:keeprank1}, $x$ is kept for all future generations, we need at most $\frac{1}{p_{1}p_{2}}$ iterations to obtain $0^n$. With fair selection, we have $p_1 = 1$ and with uniform selection, $p_1 = 1-(1-\frac{1}{N})^N \geq 1 - \frac{1}{e}$. On the other hand, $p_2 = (\frac{1}{n})^{k}(1-\frac{1}{n})^{n-k}\geq \frac{1}{en^k}$. So the expected number $\frac{1}{p_{1}p_{2}}$ of iterations to obtain $0^n$ is bounded by $en^k$ if using fair selection, and by $\frac{e^2}{e-1}n^k$ if using uniform selection. The case for obtaining $1^n$ is symmetrical. Therefore, the expected total number of iterations needed to cover the extremal values of the Pareto front is at most $2en^k$ if using fair selection, and $2\frac{e^2}{e-1}n^k$ if using uniform selection.
\end{proof}

Combining the lemmas, we immediately obtain the runtime guarantee.
\begin{theorem}\label{thm:basic}
Using population size $N \geq 4(n-2k+3)$, selecting parents using fair or uniform selection, and mutating using bit-wise mutation, the NSGA\nobreakdash-II needs in expectation at most $(1+o(1))KNn^k$ fitness evaluations to cover the entire Pareto front of the \onejumpzerojump benchmark, where $K=2e$ for fair selection and $K=2\frac{e^2}{e-1}$ for uniform selection.
\end{theorem}

In the above result, we have given explicit values for the leading constant~$K$ to show that it is not excessively large, but we have not tried to optimize this constant. In fact, it is easy to see that the $2$ could be replaced by $1.5$ by taking into account that the expected time to find the first extremal point is only half the time to find a particular extremal point. Since we have no non-trivial lower bounds at the moment, we find it too early to optimize the constants.

We now turn to the case where the mutating parents are chosen using one of two ways of \emph{binary tournaments}, namely, via $N$ independent tournaments and the two-permutation tournament scheme.

\begin{theorem}\label{thm:bin}
Using population size $N \geq 4(n-2k+3)$, selecting parents using $N$ independent tournaments or the two-permutation tournament scheme, and mutating using bit-wise mutation, the NSGA\nobreakdash-II takes in expectation at most $(1+o(1))KNn^k$ fitness evaluations to cover the entire Pareto front of the \onejumpzerojump benchmark, where $K=2\frac{e^2}{e-1}$ if using $N$ independent tournaments, and $K=\frac{8}{3}e$ if using the two-permutation tournament scheme.
\end{theorem}

The proof of this result follows the outline of the proof of Theorem~\ref{thm:basic}, but needs some technical arguments from~\cite{ZhengLD22} on the probability that an individual next to an uncovered spot on the Pareto front is chosen to be mutated.

\begin{proof}[Proof of Theorem~\ref{thm:bin}]
Since we aim for an asymptotic statement, we shall always assume that $n$, and consequently $N$, are sufficiently large. 

From Lemma \ref{lem:stage1} we know that stage $1$ takes at most $e(\frac{4k}{3})^k$ iterations in expectation.

The analysis for stage $2$ is similar to that in Lemma \ref{lem:stage2}, but due to the different parent selection method, we need to find a new lower bound for $p$, the probability of selecting a suitable parent. Consider an iteration $t$ of stage $2$, let $V=f(P_t)$, and let $v_{\min}=\min\{f_1(x)\mid x\in P_t\}$ and $v_{\max}=\max\{f_1(x)\mid x\in P_t\}$. Also, define $V^{+}=\{(v_1, v_2)\in V \mid \exists y \in P_t:(f_1(y), f_2(y)) = (v_1 + 1, v_2 - 1)\}$ and $V^{-}=\{(v_1, v_2)\in V \mid \exists y \in P_t:(f_1(y), f_2(y)) = (v_1 - 1, v_2 + 1)\}$. Then by Lemma~$3$ of \cite{ZhengLD22}, for any $(v_1, v_2)\in V\backslash(V^{+}\cap V^{-})$, there is at least one individual $x\in P_t$ with $f(x)=(v_1, v_2)$ and $\cDis(x)\geq \frac{2}{v_{\max}-v_{\min}}$. We call such individuals ``desirable individuals''. By Lemma $4$ of \cite{ZhengLD22}, for any $(v_1, v_2)\in V^{+}\cap V^{-}$, there are at most two individuals $x\in P_t$ with $f(x)=(v_1, v_2)$ and $\cDis(x)\geq \frac{2}{v_{\max}-v_{\min}}$. Since $|F_I^*\backslash\{(2k, n),(n, 2k)\}| = n-2k-1$, there are at most $2(n-2k-1)$ such individuals. By the proof of Lemma~\ref{lem:keeprank1}, there are at most $16$ individuals with objectives values $(k, n+k)$, $(2k, n)$, $(n, 2k)$, or $(n+k, k)$ and positive crowding distance. Therefore, there are at least $(N-1)-2(n-2k-1)-16\geq \frac{N}{2}-9$ individuals that, if chosen as the opponent of a desirable individual, the winner of the binary tournament is a desirable individual. Thus, with $N$ independent tournaments, the probability that a desirable individual participates and wins in a binary tournament is at least $p=1-(1-\frac{1}{N}\frac{\frac{N}{2}-9}{N-1})^N=1-(1-\frac{1}{2N}\frac{N-18}{N-1})^N$, which is an increasing function for $N > 0$. So for $N \geq 4(n-2k+3) \geq 28$, we have $p\geq 0.17$. With the two-permutation tournament scheme, the probability is at least $\frac{\frac{N}{2}-9}{N-1}=\frac{N-18}{2(N-1)} \geq 0.18$ for $N \geq 28$. Following the proof of Lemma \ref{lem:stage2}, the expected number of iterations when using either tournament scheme is at most $\frac{en}{\Omega(1)}(\sum_{e-1}^{n-v}\frac{1}{i}+\sum_{i=k}^{v}\frac{1}{i})=O(n\log n)$, where $v$ is the number of $1$-bits in the first $x\in S_I^*$ found by the algorithm.

The analysis for stage $3$ is similar to that in Lemma~\ref{lem:stage3}. Let $x$ be an individual in $P_t$ such that $|x|_1 = k$. Suppose $0^n \notin P_t$, then $x$ has the largest $f_2$ value in rank-$1$ individuals of $P_t$. Therefore $\cDis(x) = \infty$ and if $x$ is selected to participate in a binary tournament, in the worst case it wins with probability $\frac{1}{2}$. If using $N$ independent binary tournaments, the probability of $x$ chosen as a mutating parent is $1-(1-\frac{1}{2}\frac{2}{N})^N \geq 1-\frac{1}{e}$. With the two-permutation tournament scheme, the probability is $1-(\frac{1}{2})^2 = \frac{3}{4}$. Following the proof of Lemma~\ref{lem:stage3}, the number of iterations to obtain $0^n$ in stage $3$ is at most $\frac{1}{(1 - \frac{1}{e})\frac{1}{en^k}} = \frac{e^2}{e-1}n^k$ using $N$ independent tournaments, and $\frac{1}{\frac{3}{4}\frac{1}{en^k}}=\frac{4}{3}en^k$ under the two-permutation tournament scheme. The case for covering $1^n$ is symmetrical. So in total stage~$3$ needs in expectation at most $2\frac{e^2}{e-1}n^k$ iterations using $N$ independent tournaments, and $\frac{8}{3}en^k$ iterations under the two-permutation tournament scheme.
\end{proof}

\subsection{Runtime Analysis for the NSGA-II Using Fast Mutation}\label{ssec:fast}

We now consider the NSGA\nobreakdash-II with heavy-tailed mutation, i.e., the mutation operator proposed in~\cite{DoerrLMN17} and denoted by $\mutbeta$ in \cite{DoerrZ21aaai}, a work from which we shall heavily profit in the following. We note that fast mutation, that is, choosing the mutation strength from a heavy-tailed distribution, was shown to lead to good performances in many other works as well~\cite{FriedrichQW18,FriedrichGQW18,QuinzanGWF21,WuQT18,AntipovBD20gecco,AntipovBD20ppsn,AntipovD20ppsn,AntipovBD21gecco,DoerrZ21aaai,CorusOY21,DoerrR22}.

Let $\beta > 1$ be a constant (typically below $3$). Let $D^{\beta}_{n/2}$ be the distribution such that if a random variable $X$ follows the distribution, then $\Pr[X=\alpha] = (C^{\beta}_{n/2})^{-1}\alpha^{-\beta}$ for all $\alpha \in [1..n/2]$, where $n$ is the size of the problem and $C^{\beta}_{n/2} := \sum_{i=1}^{n/2}i^{-\beta}$.  In an application of the mutation operator $\mutbeta$, first an $\alpha$ is chosen according to the distribution $D^{\beta}_{n/2}$ (independent from all other random choices of the algorithm) and then each bit of the parent is flipped independently with probability~$\alpha/n$. Let $x\in \{0, 1\}^n$, $y\sim\mutbetaarg(x)$, and $H(x,y)$ denote the Hamming distance between $x$ and $y$. Then, by Lemma~$13$ of~\cite{DoerrZ21aaai}, we have
\[P^{\beta}_j := \Pr[H(x,y)=j] = \begin{cases} (C^{\beta}_{n/2})^{-1}\Theta(1) & \text{for } j=1;\\ (C^{\beta}_{n/2})^{-1}\Omega(j^{-\beta}) & \text{for } j\in[2..n/2]. \end{cases}\]

\begin{theorem}{\label{thm:heavy}}Using population size $N \geq 4(n-2k+3)$, any one of the four parent selection methods, and mutating with the $\mutbeta$ operator, the NSGA\nobreakdash-II takes at most $(1+o(1))\frac{1}{P^\beta_k}NK\binom{n}{k}$ fitness evaluations in expectation to cover the entire Pareto front of the \onejumpzerojump benchmark, where $K=2$ for fair selection, $K=\frac{2e}{e-1}$ for uniform selection and selection via $N$ independent binary tournaments, and $K=\frac{8}{3}$ for the two-permutation binary tournament scheme.
\end{theorem}

\begin{proof}[Proof of Theorem~\ref{thm:heavy}]
The analysis for stage $1$ is similar to the proof of Lemma~\ref{lem:stage1}. Now the probability of obtaining an individual with $k$ $1$-bits from an individual with $i < k$ $1$-bits is $p_1 = \frac{{n-i \choose k-i}}{{n \choose k-i}}P_{k-i}^\beta=\frac{(n-i)!(n-k+i)!}{(n-k)!n!}P_{k-i}^\beta > (\frac{n-k}{n})^iP_{k-i}^\beta > (\frac{3}{4})^kP_{k-i}^\beta$. Similarly, the probability of obtaining an individual with $n-k$ $1$-bits from an individual with $i > n-k$ $1$-bits is $p_2 = \frac{{i \choose i-(n-k)}}{{n \choose i-(n-k)}}P_{i-(n-k)}^\beta=\frac{i!(2n-i-k)!}{(n-k)!n!}P_{i-(n-k)}^\beta > (\frac{n-k}{n})^{n-i}P_{k-i}^\beta > (\frac{3}{4})^kP_{k-i}^\beta$. Since $\beta > 1$, $C^{\beta}_{n/2} \leq \frac{\beta}{\beta -1}$. Consequently $p_1, p_2 > (\frac{3}{4})^k\frac{\beta-1}{\beta}\Omega(k^{-\beta})$. So the number of iterations needed in expectation is at most $O((\frac{4}{3})^{k}k^\beta)$.

The analysis of stage $2$ follows the one in Lemma~\ref{lem:stage2}. Now, from an individual with $i$ $1$-bits, $p_{i}^{+}$, the probability of generating an individual with one more $1$ is $\frac{n-i}{n}P^{\beta}_1$, and $p_{i}^{-}$, the probability of generating an individual with one more $0$ is $\frac{i}{n}P^{\beta}_1$. So the number of iterations needed in expectation for stage $2$ is at most
\[\sum_{i=v}^{n-k} \frac{1}{pp_{i}^{+}} + \sum_{i=k}^{v} \frac{1}{pp_{i}^{-}} \leq \frac{2en}{(e-1)P^{\beta}_1}(1+\ln{n})=O(n\log n).\]

The analysis of stage $3$ is again similar to the one in Lemma~\ref{lem:stage3} except that now we profit from the much larger probability to flip $k$ bits. The probability of selecting an individual with $k$ $1$-bits as a mutating parent, $p_1$, remains unchanged, and thus $p_1=1$ with fair selection and $p_1 \geq 1-\frac{1}{e}$ with uniform selection. Now, the probability of generating an individual with $0$ $1$-bits from an individual with $k$ $1$-bits, $p_2 = \frac{P^{\beta}_{k}}{\binom{n}{k}}$. So the number of iterations needed in expectation to find $0^n$ is at most $\frac{\binom{n}{k}}{P^{\beta}_{k}}$ with fair selection, and $\frac{e\binom{n}{k}}{(e-1)P^{\beta}_{k}}$ with uniform selection. Finding $1^n$ is the same. So the total number of iterations needed in expectation is at most $\frac{2\binom{n}{k}}{P^{\beta}_{k}}$ for fair selection and $\frac{2e\binom{n}{k}}{(e-1)P^{\beta}_{k}}$ for uniform selection.

Combining the three stages, the number of fitness evaluations needed in total is at most $(1+o(1))\frac{1}{P^\beta_k}K\binom{n}{k}$ in expectation, where $K=2$ for fair selection, $K=\frac{2e}{e-1}$ for uniform selection.

With binary tournaments, the analysis is the same with the exception that the probabilities of selecting a particular individual are different. According to the proof of Theorem~\ref{thm:bin}, using $N$ independent binary tournaments, the probability of selecting a desirable individual in stage $2$ becomes at least $0.17$, and the probability of selecting a desirable individual in stage $3$ is at least $1-\frac{1}{e}$. Accordingly, $K=\frac{2e}{e-1}$. Moreover, with the two-permutation binary tournament scheme, in stage~$2$ the probability that a desirable parent is selected is at least $0.18$, and the probability in stage $3$ is at least $\frac{3}{4}$, which makes $K$ equal to $\frac{8}{3}$. 
\end{proof}

Noting that $\binom{n}{k}$ is by a factor of $k^{\Omega(k)}$ smaller than $n^k$, whereas $1/P^\beta_k$ is only $O(k^{\beta})$, we see that the runtime guarantee for the heavy-tailed operator is by a factor of $k^{\Omega(k)}$ stronger than our guarantee for bit-wise mutation. Without a lower bound on the runtime in the bit-wise setting, we cannot claim that the heavy-tailed algorithm is truly better, but we strongly believe so (we do not see a reason why the NSGA\nobreakdash-II with bit-wise mutation should be much faster than what our upper bound guarantees). 

We note that it is easy to prove a lower bound of $\Omega(n^k)$ for the runtime of the NSGA\nobreakdash-II with bit-wise mutation (this is a factor of $N$ below our upper bound, which stems from pessimistically assuming that in each iteration, $N$ times a parent is selected that has a $\Theta(n^{-k})$ chance of generating an extremal point of the Pareto front). For $k$ larger than, say, $\log(N)$, this weak lower bound would suffice to show that the heavy-tailed NSGA\nobreakdash-II is asymptotically faster. We spare the details and hope that at some time, we will be able to prove tight lower bounds for the NSGA\nobreakdash-II.

When $k \le \sqrt n$, the runtime estimates above can be estimated further as follows. In~\cite{DoerrZ21aaai}, it was shown that \[P^{\beta}_i \geq \begin{cases} \frac{\beta-1}{e\beta} & \text{for } i=1;\\ \frac{\beta-1}{4\sqrt{2\pi}e^{8\sqrt{2}+13}\beta}i^{-\beta} & \text{for } i\in[2..\lfloor\sqrt{n}\rfloor]. \end{cases}\]
Also, for $k \le \sqrt n$, a good estimate for the binomial coefficient is $\binom{n}{k} \le \frac{n^k}{k!}$ (losing at most a constant factor of $e$, and at most a $(1+o(1))$-factor when $k = o(\sqrt n)$. Hence the runtime estimate from Theorem~\ref{thm:heavy} for $k \le \sqrt n$ becomes
\[
(1+o(1)) K \frac{4\sqrt{2\pi}e^{8\sqrt{2}+13}\beta}{\beta-1} N k^\beta \frac{n^k}{k!},
\]
which is a tight estimate of the runtime guarantee of Theorem~\ref{thm:heavy} apart from constants independent of $n$ and $k$. In any case, this estimate shows that for moderate values of $k$, our runtime guarantee for the heavy-tailed NSGA\nobreakdash-II is better by a factor of $\Theta(k! k^{-\beta})$, which is substantial already for small values of~$k$.

\section{Experiments}\label{sec:exp}

To complement our theoretical results, we also experimentally evaluate the runtime of the NSGA\nobreakdash-II algorithm on the \ojzj benchmark. 

\subsection{Settings}\label{ssec:settings}

We implemented the algorithm as described in Section~\ref{sec:prelim} in Python (the code can be found at \url{https://github.com/deliaqu/NSGA-II}). We use the following settings. 
\begin{itemize}
    \item Problem size $n$: $20$ and $30$.
    \item Jump size $k$: $3$.
    \item Population size $N$: In our theoretical analysis, we have shown that with $N=4(n-2k+3)$, the algorithm is able to recover the entire Pareto front. To further explore the effect of the population size, we conduct experiments with this population size, with half this size, and with twice this size, that is, for  $N\in \{2(n-2k+3), 4(n-2k+3), 8(n-2k+3)\}$.
    \item Parent selection: For simplicity, we only experiment with using $N$ independent binary tournaments.
    \item Mutation operator: Following our theoretical analysis, we consider two mutation operators, namely bit-wise mutation (flipping each bit with probability~$\frac{1}{n}$) and fast mutation, that is, the heavy-tailed mutation operator $\mutbeta$. We set $\beta$ to be $1.5$.
    \item Number of independent repetitions per setting: $50$. This number is a compromise between the longer runtimes observed on a benchmark like \ojzj and the not very concentrated runtimes (for most of our experiments, we observed a corrected sample standard deviation between 50\% and 80\% of the mean, which fits to our intuition that the runtimes are dominated by the time to find the two extremal points of the Pareto front).
\end{itemize}

\subsection{Experimental Results}

\addtolength{\tabcolsep}{4pt}
\begin{table}
\centering
\caption{Average runtime of the NSGA-II and the GSEMO with bit-wise mutation and heavy-tailed mutation operator on the \ojzj benchmark with ${k=3}$.\\}
\label{tab1}
\begin{tabular}{|c|c|c|c|c|c|}
\hline
& \multicolumn{2}{|c|}{$n=20$} & \multicolumn{2}{|c|}{$n=30$} \tabularnewline
\hline
& Bit-wise & HT & Bit-wise & HT \\
\hline
$N=2(n-2k+3)$ & $264932$ & $178682$ & $1602552$ & ~$785564$ \\
$N=4(n-2k+3)$ & $366224$ & $188213$ & $1777546$ & $1080458$ \\
$N=8(n-2k+3)$ & $529894$ & $285823$ & $2836974$ & $1804394$\\
\hline
GSEMO & $511365$ & $215001$ & $2654620$ & $1422455$\\
\hline
\end{tabular}
\end{table}
Table \ref{tab1} contains the average runtime (number of fitness evaluations done until the full Pareto front is covered) of the NSGA\nobreakdash-II algorithm when using bit-wise mutation and the heavy-tailed mutation operator. The most obvious finding is that the heavy-tailed mutation operator already for these small problem and jump sizes gives significant speed-ups.

While our theoretical results are valid only for $N \ge 4(n-2k+3)$, our experimental data suggests that also with the smaller population size $N=2(n-2k+3)$ the algorithm is able to cover the entire Pareto front of the \ojzj benchmark. We suspect that this is because even though theoretically for each objective value there could be $4$ individuals in each generation having positive crowding distances, empirically this happens relatively rarely and for each objective value the expected number of individuals with positive crowding distances is closer to $2$. We also note that with a larger population, e.g., $N=8(n-2k+3)$, naturally, the runtime increases, but usually by significantly less than a factor of two. This shows that the algorithm is able to profit somewhat from the larger population size. 

For comparison, we also conducted experiments with the global simple evolutionary
multi-objective optimizer (GSEMO)~\cite{Giel03}. The results can be found on the last row of Table~\ref{tab1}. We note that with bit-wise mutation, the GSEMO has a similar runtime as the NSGA-II when $N=8(n-2k+3)$. In other words, when using the population size $N=4(n-2k+3)$, for which our runtime guarantee applies, or smaller, but still efficient population sizes, the NSGA-II clearly outperforms the GSEMO. The heavy-tailed operator speeds up the GSEMO more than the NSGA-II, but still the NSGA-II remains faster with the two smaller population sizes.

\subsection{Crossover}

Besides fast mutation, two further mechanisms were found that can speed up the runtime of evolutionary algorithms on (single-objective) jump functions, namely the stagnation-detection mechanism of Rajabi and Witt~\cite{RajabiW20,RajabiW21evocop,RajabiW21gecco,DoerrR22} and crossover~\cite{JansenW02,DangFKKLOSS16,DangFKKLOSS18,AntipovBD22}. We are relatively optimistic that stagnation detection, as when used with the global SEMO algorithm~\cite{DoerrZ21aaai}, can provably lead to runtime improvements, but we recall from~\cite{DoerrZ21aaai} that the implementation of stagnation detection is less obvious for MOEAs. For that reason, we ignore this approach here and immediately turn to crossover, given that no results proving a performance gain from crossover for the NSGA\nobreakdash-II exist, and there are clear suggestions on how to use it in~\cite{DangFKKLOSS18}.

Inspired by~\cite{DangFKKLOSS18}, we propose and experimentally analyze the following variant of the NSGA\nobreakdash-II. The basic algorithm is as above.
In particular, we also select $N$ parents via independent tournaments. We partition these into pairs. For each pair, with probability $90\%$, we generate two intermediate offspring via a $2$-offspring uniform crossover (that is, for each position independently, with probability $0.5$, the first child inherits the bit from the first parent, and otherwise from the second parent; the bits from the two parents that are not inherited by the first child make up the second child). We then perform bit-wise mutation on these two intermediate offspring. With the remaining $10\%$ probability, mutation is performed directly on the two parents.

\begin{table}[t]
\centering
\caption{Average runtime of the NSGA\nobreakdash-II with bit-wise mutation and crossover on the \ojzj benchmark with ${k=3}$.\\}\label{tab2}

\begin{tabular}{|r|r|r|r|}
\hline
& $n=20$ & $n=30$ & $n=40$ \\
\hline
$N=2(n-2k+3)$ & $68598$ & $265993$ & $773605$\\
$N=4(n-2k+3)$ & $45538$ & $205684$ & $510650$\\
$N=8(n-2k+3)$ & $68356$ & $316500$ & $635701$\\
\hline
\end{tabular}
\end{table}

Table \ref{tab2} contains the average runtimes for this algorithm. We observe that crossover leads to massive speed-ups (which allows us to also conduct experiments for problem size $n=40$). More specifically, comparing the runtimes for $n=30$ and bit-wise mutation (which is fair since the crossover version also uses this mutation operator), the crossover-based algorithm only uses between 8\% and 15\% percent of the runtime of the mutation-only algorithm.

We note that different from the case without crossover, with $N=2(n-2k+3)$, the algorithm consistently takes more time than with $N=4(n-2k+3)$. We suspect that the smaller population size makes it less likely that the population contains two parents from which crossover can create a profitable offspring.  We note that in~\cite{DangFKKLOSS18}, the best speed-up from crossover on a jump function with jump size $k=3$ was obtained for a population size of $\Theta(n)$ -- since this result was for the single-objective jump function, the whole population is concentrated on the fitness value of the local optima. As our larger runtimes for $N = 8(n-2k+3)$ suggest, for the NSGA\nobreakdash-II comparably smaller population sizes (having in average a constant number of individuals on each objective value) are more profitable.

\begin{table}[t]
\centering
\caption{Average runtime of the NSGA\nobreakdash-II with heavy-tailed mutation and crossover on the \ojzj benchmark with ${k=3}$.\\}\label{tab3}

\begin{tabular}{|r|r|r|r|}
\hline
& $n=20$ & $n=30$ & $n=40$ \\
\hline
$N=2(n-2k+3)$ & $52874$ & $234005$ & $695998$\\
$N=4(n-2k+3)$ & $60626$ & $248681$ & $696766$\\
$N=8(n-2k+3)$ & $103741$ & $474932$ & $1504574$\\
\hline
\end{tabular}
\end{table}

Table~\ref{tab3} contains the average runtimes when using heavy-tailed mutation with crossover. Similarly to the case of using bit-wise mutation, the selected parents are divided into pairs and with $90\%$ chance uniform crossover is performed on the pair followed by heavy-tailed mutation, and with $10\%$ chance only mutation is performed. We observe that when compared with using bit-wise mutation and crossover, we only see improvement in the case where $N=2(n-2k+3)$. We do not yet understand why the two mechanisms to speed up the runtimes in combination do not consistently bring further improvements. We believe that a much more thorough understanding of the working principles of crossover in this application is needed to make further progress here.

\section{Conclusions and Future Works}

In this first mathematical runtime analysis of the NSGA\nobreakdash-II on a bi-ojective multimodal problem, we have shown that the NSGA\nobreakdash-II with a sufficient population size performs well on the \ojzj benchmark and profits from heavy-tailed mutation, all comparable to what was shown before for the GSEMO algorithm. 

Due to the more complicated population dynamics of the NSGA\nobreakdash-II, we could not yet prove an interesting lower bound. For this, it would be necessary to understand how many individuals with a particular objective value are in the population. Note that this number is trivially one for the GSEMO, which explains why for this algorithm lower bounds could be proven~\cite{DoerrZ21aaai}. Understanding better the population dynamics of the NSGA\nobreakdash-II and then possibly proving good lower bounds is an interesting and challenging direction for future research. 

A second interesting direction is to analyze theoretically how the NSGA\nobreakdash-II with crossover optimizes the \ojzj benchmark. Our experiments show clearly that crossover can lead to significant speed-ups here. Again, we currently do not have the methods to analyze this algorithm, and we speculate that a very good understanding of the population dynamics is necessary to solve this problem. We note that the only previous work~\cite{BianQ22ppsn} regarding the NSGA\nobreakdash-II with crossover does not obtain faster runtimes from crossover. Besides that work, we are only aware of two other runtime analyses for crossover-based MOEAs, one for the multi-criteria all-pairs shortest path problem~\cite{NeumannT10}, the other also for classic benchmarks, but with an initialization that immediately puts the extremal points of the Pareto front into the population~\cite{QianYZ13}. So it is unlikely that previous works can be used to analyze the runtime of the crossover-based NSGA\nobreakdash-II on \ojzj.

\subsection*{Acknowledgment}

This work was supported by a public grant as part of the
Investissements d'avenir project, reference ANR-11-LABX-0056-LMH,
LabEx LMH.
\bibliographystyle{alphaurl}
\bibliography{ich_master,alles_ea_master}

}
\end{document}